\documentclass[11pt]{article}

\usepackage[margin=1in]{geometry}
\usepackage{amsmath,amssymb,amsfonts,amsthm}
\usepackage{booktabs}
\usepackage{tabularx}
\usepackage{array}
\usepackage{enumitem}
\usepackage{xcolor}
\usepackage{microtype}
\usepackage[numbers,sort&compress]{natbib}
\usepackage{hyperref}
\usepackage[nameinlink,capitalize,noabbrev]{cleveref}

\hypersetup{
    colorlinks=true,
    linkcolor=blue!60!black,
    citecolor=green!40!black,
    urlcolor=blue!70!black
}

\newtheorem{theorem}{Theorem}[section]
\newtheorem{proposition}[theorem]{Proposition}
\newtheorem{corollary}[theorem]{Corollary}
\newtheorem{assumption}[theorem]{Assumption}
\newtheorem{definition}[theorem]{Definition}
\newtheorem{remark}[theorem]{Remark}

\newcommand{\E}{\mathbb{E}}

\newcommand{\cA}{\mathcal{A}}
\newcommand{\cD}{\mathcal{D}}
\newcommand{\cE}{\mathcal{E}}
\newcommand{\cF}{\mathcal{F}}
\newcommand{\cH}{\mathcal{H}}
\newcommand{\cI}{\mathcal{I}}
\newcommand{\cL}{\mathcal{L}}
\newcommand{\cM}{\mathcal{M}}
\newcommand{\cO}{\mathcal{O}}
\newcommand{\cP}{\mathcal{P}}
\newcommand{\cQ}{\mathcal{Q}}
\newcommand{\cS}{\mathcal{S}}
\newcommand{\cT}{\mathcal{T}}
\newcommand{\cW}{\mathcal{W}}
\newcommand{\cX}{\mathcal{X}}
\newcommand{\cY}{\mathcal{Y}}
\newcommand{\cZ}{\mathcal{Z}}

\newcommand{\PiPol}{\boldsymbol{\Pi}}
\newcommand{\EnvClass}{\mathfrak{M}}
\newcommand{\CapClass}{\mathfrak{C}}
\newcommand{\LearnClass}{\mathfrak{L}}
\newcommand{\Kern}{\mathfrak{K}}

\newcommand{\pimu}{\pi_{\mu}}

\newcommand{\TV}{\operatorname{TV}}
\newcommand{\KL}{D_{\mathrm{KL}}}
\newcommand{\doop}{\operatorname{do}}
\newcommand{\argmaxop}{\operatorname*{arg\,max}}
\newcommand{\argminop}{\operatorname*{arg\,min}}
\newcommand{\Ability}{\operatorname{Ability}}
\newcommand{\Reg}{\operatorname{Reg}}
\newcommand{\indep}{\perp\!\!\!\perp}

\title{
\textbf{On the Capability Separation Between World-Model Policy Learning
and Imitated World-Action Models}
}

\author{
Yang Yu\\
Nanjing University\\
\texttt{yuy@nju.edu.cn}
}

\date{}

\begin{document}

\maketitle

% ============================================================
% Abstract
% ============================================================
\begin{abstract}
World-action models predict a future world outcome and then infer an action
associated with that outcome. This factorization differs from direct action
prediction and may provide practical benefits through video pretraining,
temporal supervision, and future-conditioned representations. It remains
unclear, however, whether these benefits imply a stronger control capability
when both systems are trained only from the same observational
demonstrations.

We study this question at the level of controller classes and population
learning targets. A population objective is an expectation under the true
demonstration distribution rather than an empirical average over a finite
dataset. The main analysis therefore abstracts away from finite-sample
estimation error, optimization error, and model misspecification. We define
the control-capability class of a policy class through the closed-loop
trajectory distributions that its members can induce.

We compare a direct behavior-cloning policy $\pi_{\mathrm{A}}$, an imitated
world-action policy $\pi_{\mathrm{WA}}$, and a policy
$\pi_{\mathrm{WM}}$ optimized using an action-conditioned world model. Every
world-action controller can be flattened into a direct stochastic policy
with the same closed-loop behavior. Under unrestricted stochastic-kernel
classes, the two architectures consequently have the same external
control-capability class. Moreover, under realizability, exact population
optimization, common deployment information, and distribution-preserving
world-action deployment,
\[
    \pi_{\mathrm{WA}}^*(a\mid h)
    =
    \pi_{\mathrm{A}}^*(a\mid h)
    =
    \pi_\mu(a\mid h),
\]
where $\pi_\mu(a\mid h):=p_\mu(a\mid h)$ is the observational
behavior-action conditional. Under an additional behavior-sufficiency
condition, these policies induce the same closed-loop trajectory distribution
as the demonstrator.

World-model policy learning differs in its decision rule and, when
observational identification fails, in the information required to select a
policy. It predicts the consequences of specified candidate actions,
\[
    P(Y\mid H,\doop(A=a)),
\]
and compares those consequences through a control objective. We establish an
irreducible action-specific prediction gap for future predictors that do not
condition on the current candidate action. We then characterize when an exact
world-action joint determines a forward model and show why causal
identification and action-clamping deployment remain separate requirements.
Finally, we construct an environment family in which every observational
learner has positive worst-case regret, whereas one informative action
intervention permits zero regret.

The relevant distinction is therefore not whether a controller predicts the
future, but whether it identifies and compares futures under specified
alternative actions.
\end{abstract}

% ============================================================
\section{Introduction}
\label{sec:introduction}
% ============================================================

Future prediction has become an important component of embodied
decision-making systems. Instead of mapping the current observation,
instruction, and interaction history directly to an action, a model may first
predict a future image, video, state, or latent representation and then infer
an action associated with that future. Such world-action models can exploit
large-scale video data, dense temporal supervision, and representations of
geometry, contact, and motion
\cite{tian2025pidm,ye2026wam}.

This development raises a basic question: if a world-action model and a
direct action model are trained only from the same observational
demonstrations, does future prediction give the world-action model a stronger
control capability, or does it provide a different factorization for learning
the same behavior policy?

Several effects are easily conflated in empirical comparisons. A world-action
system may use additional pretraining data, receive future-frame supervision,
have a more suitable inductive bias, or be easier to optimize. It may
therefore outperform a direct behavior-cloning baseline even if both methods
have the same ideal action target. Conversely, a system that predicts the
consequences of \emph{specified candidate actions} uses a decision interface
that standard behavior-reproduction deployment does not provide. When those
action effects are not observationally identified, it also requires
additional information, causal assumptions, or interventions. Improved task
performance alone does not distinguish among these possibilities.

A direct behavior-cloning policy predicts the next action $A$ from the
controller's available history $H$:
\[
    H\longrightarrow A.
\]
We denote this policy by $\pi_{\mathrm{A}}$. Under a realizable population
objective, it recovers the observational behavior-action conditional
$p_\mu(A\mid H)$. Here, a population objective is an expected loss under the
true demonstration distribution rather than an empirical average over a
finite dataset. Many vision-language-action systems are trained primarily
through behavior-cloning objectives, but the term
``vision-language-action'' describes an interface or architecture and does
not by itself determine the learning objective. The claims in this paper
concern systems trained through observational behavior reproduction.

An imitated world-action model uses a different factorization:
\[
    H\longrightarrow Y\longrightarrow A,
\]
where $Y$ is a future observation, state, video chunk, or latent outcome. We
denote its externally executed policy by $\pi_{\mathrm{WA}}$. When the model
is trained to fit observational tuples $(H,A,Y)$, its future component models
a future under the behavior distribution, and its inverse component predicts
an action associated with that future. The future variable can provide useful
training supervision even though it is not yet observed when the action must
be selected at deployment.

World-model policy learning uses future prediction differently. An
action-conditioned world model represents
\[
    (H,A)\longrightarrow Y
\]
and is intended to answer:
\begin{quote}
What outcome would occur if the controller selected this candidate action?
\end{quote}
Formally, $\doop(A=a)$ denotes an intervention that externally sets the
current action to $a$. A planner, value estimator, verifier, or policy
optimizer can compare the predicted consequences of alternative actions and
produce a policy $\pi_{\mathrm{WM}}$. This is the decision mechanism
underlying model-based reinforcement learning and model-predictive control
\cite{luo2024mbrlsurvey}.

The three paradigms can be summarized as
\[
\begin{array}{lll}
\pi_{\mathrm{A}}:
& H\rightarrow A,
& \text{direct behavior reproduction},\\[0.25em]
\pi_{\mathrm{WA}}:
& H\rightarrow Y\rightarrow A,
& \text{structured behavior reproduction},\\[0.25em]
\pi_{\mathrm{WM}}:
& (H,A)\rightarrow Y\rightarrow \text{utility},
& \text{action comparison and optimization}.
\end{array}
\]

All three ultimately execute stochastic mappings from histories to actions.
Their distinction therefore cannot be established merely by comparing their
final interfaces. A comparison between trained systems mixes at least four
questions:

\begin{enumerate}[leftmargin=2em]
    \item \textbf{External expressivity:}
    which closed-loop behaviors can the controller architecture represent?

    \item \textbf{Population target:}
    which conditional distribution minimizes the expected training objective
    under the true demonstration distribution?

    \item \textbf{Information and identification:}
    which action effects are determined by the information available during
    training?

    \item \textbf{Learning error:}
    how do finite data, restricted model classes, optimization, and
    computation affect the learned controller?
\end{enumerate}

This paper focuses on the first three questions. We compare policy and learner
classes rather than the test errors of particular finite neural networks. In
the main exact results, the model classes contain the relevant true
conditionals and the population objectives are optimized exactly. The
approximate result is a sensitivity statement rather than a finite-sample
learning bound.

We define a control-capability class through the closed-loop trajectory
distributions representable by a policy class. We separately define
observational and interventional learner classes according to the information
available to the learning rule. The distinction established below is always
one of deployment and decision rule and, when observational identification
fails, also one of available information. It is not a separation in
unrestricted external policy expressivity.

\paragraph{Main results.}
The analysis gives five main conclusions.

\begin{enumerate}[leftmargin=2em]
    \item \textbf{Architecture-level equivalence.}
    Every distributional world-action controller induces a direct stochastic
    action kernel with exactly the same closed-loop trajectory distribution.
    Conversely, every direct stochastic policy has a degenerate world-action
    representation. Direct policies and world-action controllers therefore
    have the same external control-capability class when both range over
    unrestricted stochastic kernels.

    \item \textbf{Population equivalence under observational imitation.}
    Under realizability, exact population optimization, common observational
    data, and distribution-preserving deployment,
    \[
        \pi_{\mathrm{WA}}^*(a\mid h)
        =
        \pi_{\mathrm{A}}^*(a\mid h)
        =
        \pimu(a\mid h)
    \]
    for $p_\mu(H)$-almost every history, where
    $\pimu(a\mid h):=p_\mu(a\mid h)$. Under behavior sufficiency, this
    equality extends to complete closed-loop trajectory distributions. We
    also give an approximate version that bounds trajectory-level
    disagreement in terms of errors in future prediction, inverse-action
    prediction, and direct action prediction.

    \item \textbf{Action-specific prediction separation.}
    A future predictor that does not condition on the current candidate
    action must average over candidate-action branches. We characterize its
    irreducible interventional prediction error by
    \[
        I_{\nu,\rho}^{\mathrm{int}}(A;Y\mid H),
    \]
    which is positive whenever alternative actions induce different outcome
    distributions on the evaluation support.

    \item \textbf{Separation of representation, identification, and
    deployment.}
    An exact world-action joint determines the observational conditional
    $p_\mu(Y\mid H,A)$ wherever the behavior distribution has positive action
    support. Under consistency, conditional exchangeability, and positivity,
    this conditional equals
    \[
        P(Y\mid H,\doop(A)).
    \]
    The identified joint can then support action comparison if candidate
    actions are explicitly clamped. Standard future-then-inverse deployment,
    however, marginalizes the same joint and recovers the behavior policy.

    \item \textbf{Decision and information separation.}
    Observational demonstrations do not identify unsupported action effects
    and need not identify action effects in the presence of hidden
    confounding. When an exact interventional model is available, model-based
    optimization can improve beyond a suboptimal demonstrator. Moreover, we
    construct a two-action environment family in which every observational
    population learner has positive worst-case regret, whereas one
    informative action intervention permits zero regret.
\end{enumerate}

These results locate the difference among the three paradigms. Direct
behavior cloning and standard world-action imitation differ in their
internal factorization but not in their unrestricted external policy class or
population action target. World-model policy learning differs in its
deployment rule and, when action effects are not observationally identified,
in the information required for policy selection. The practical benefits of
world-action modeling through representation learning, temporal supervision,
and finite-sample efficiency remain compatible with this population-level
equivalence.

\paragraph{Organization.}
\Cref{sec:related-work} reviews the most closely related work.
\Cref{sec:framework} defines the environment, population objectives, policy
classes, learner classes, and assumptions.
\Cref{sec:equivalence} establishes the architecture-level and
population-level equivalence of direct behavior cloning and imitated
world-action modeling. \Cref{sec:separation} formalizes the interventional
identification and decision requirements of world-model policy learning.
\Cref{sec:limitations-conclusion} discusses limitations and concludes.
Proofs and additional technical remarks are collected in the appendix.

% ============================================================
\section{Related Work}
\label{sec:related-work}
% ============================================================

The terms behavior cloning, world-action modeling, world modeling, and
model-based control are sometimes applied to systems with overlapping
architectural components. The distinction relevant here is whether future
prediction is used to reproduce observational behavior or to compare the
consequences of specified actions.

\subsection{Behavior cloning and imitation learning}

Behavior cloning estimates a demonstrator's conditional action distribution
from observational trajectories. A central sequential difficulty is that
action errors can move the learned policy to histories poorly covered by the
demonstrations \cite{ross2011dagger}. This has motivated trajectory-level
analyses of imitation error \cite{xu2020error,xu2022errorbounds}, imitation
with misspecified simulators \cite{jiang2020misspecified}, and learning from
imperfect demonstrations \cite{li2023imperfection,cao2024preference}.

These works primarily study finite-sample error, optimization, and
distribution shift. Our question is complementary: before those errors are
introduced, does an observational future variable change the population
action target or the externally realizable policy class?

\subsection{World-action models and world-model control}

World-action architectures use future prediction as an intermediate
representation for action inference
\cite{tian2025pidm,ye2026wam}. This factorization may exploit video
pretraining and temporal structure even when the final objective remains
behavior reproduction.

World models are commonly used for planning, policy optimization, value
expansion, and imagined experience \cite{luo2024mbrlsurvey}. Work on
model--Bellman inconsistency and reward-consistent dynamics emphasizes that
predictive accuracy and decision-relevant accuracy need not coincide
\cite{sun2023modelbellman,luo2024rewardconsistent}. Policy-conditioned and
long-horizon models further study how predictive structure interacts with
control \cite{chen2024policyconditioned,lin2025anystep}.

In visual control, future models have been used for viewpoint-invariant
prediction and world-model-based policy learning
\cite{pang2025viewinvariant,zhang2026practicalwm}. These systems illustrate
that future prediction can enter a controller through different decision
rules. Our analysis isolates the distinction between using a future as an
observational latent variable and using action-conditioned futures for
optimization.

\subsection{Causal identification and offline model-based control}

An observational conditional
\[
    p(Y\mid H,A)
\]
is not automatically equal to
\[
    P(Y\mid H,\doop(A)).
\]
The equality requires causal identification conditions such as consistency,
conditional exchangeability, and positivity
\cite{pearl2009causality,dehaan2019causal}.

Invariant action-effect models, identifiable world-model factorizations, and
counterfactual environment models seek to recover stable or causal action
effects
\cite{zhu2022invariant,chen2023counterfactual,
liu2023identifiable,zhu2025causalwm}. Model-predictive control from
observational and interventional data likewise depends on the distinction
between correlation and intervention \cite{mou2024observational}.

Offline model-based reinforcement learning faces a related support problem:
an optimized policy may select actions whose consequences are poorly
represented in the dataset
% \cite{chen2023adaptable,ran2023dataset}.  We characterize what follows from the observational
% world-action factorization
% \[
%     p(Y\mid H)p(A\mid H,Y)
% \]
% and distinguish the information represented by that joint from the decision
% rule used at deployment.

% ============================================================
\section{Framework and Scope}
\label{sec:framework}
% ============================================================

This section defines the sequential environment, demonstration distribution,
population objectives, controller classes, learner classes, and assumptions
used in the comparison.

The total-variation distance between probability measures $P$ and $Q$ is
\[
    \TV(P,Q)
    :=
    \sup_B |P(B)-Q(B)|,
\]
where the supremum is over measurable events. The Kullback--Leibler
divergence is denoted by
\[
    \KL(P\Vert Q)
    :=
    \int
    \log\left(\frac{dP}{dQ}\right)dP
\]
when $P$ is absolutely continuous with respect to $Q$, and is $+\infty$
otherwise.

We use density notation such as $p(y\mid h)$ for readability. For discrete
variables, integrals should be read as sums. For continuous variables, the
displayed expressions are interpreted using densities with respect to fixed
reference measures or, more generally, regular conditional probabilities.

\subsection{Sequential environment and trajectories}

Consider a finite-horizon partially observed controlled process
\[
    \cM
    =
    (\cS,\cO,\cA,P,\Omega,\rho_0,T),
\]
where $\cS$, $\cO$, and $\cA$ are standard Borel spaces. The initial state,
observations, and controlled transitions satisfy
\[
    S_0\sim\rho_0,
    \qquad
    O_t\sim\Omega(\cdot\mid S_t),
    \qquad
    S_{t+1}\sim P(\cdot\mid S_t,A_t).
\]

At time $t$, the controller receives
\begin{equation}
    H_t
    =
    (G,O_0,A_0,\ldots,A_{t-1},O_t)
    \in\cH_t,
    \label{eq:history}
\end{equation}
where $G$ may contain an instruction, goal, task identifier, or other
deployment-time context. The context may be fixed or sampled from a task
distribution included in the initial distribution. Time may be included in
$H_t$, and we write
\[
    \cH
    =
    \bigcup_{t=0}^{T-1}\cH_t.
\]

Let $\Kern(\cX\mid\cZ)$ denote the set of stochastic kernels from $\cZ$ to
$\cX$. A policy is an element
\[
    \pi\in\Kern(\cA\mid\cH).
\]

A complete trajectory is
\[
    \tau
    =
    (G,O_0,A_0,\ldots,A_{T-1},O_T)
    \in\cT_T.
\]
The trajectory distribution induced by policy $\pi$ in environment $\cM$ is
denoted by
\[
    P_{\cM}^{\pi}\in\cP(\cT_T),
\]
where $\cP(\cT_T)$ is the set of probability measures on the trajectory
space.

\subsection{Demonstrations, future labels, and population objectives}

At a demonstrated decision point, let

\begin{itemize}[leftmargin=2em]
    \item $H$ be the history available to the learned controller;
    \item $A$ be the next action or action chunk;
    \item $Y\in\cY$ be a future observation, state, latent, or video chunk.
\end{itemize}

The future label is available in a recorded trajectory but is not observed
when the action must be selected at deployment. It may be defined by a
measurable trajectory map
\[
    Y_t=\psi_t(\tau).
\]
Examples include $O_{t+1}$, a future image sequence, a latent future
representation, or the next $K$ observations.

Let $\mu$ denote the demonstration-generating mechanism. At each time step,
it induces
\[
    p_{\mu,t}(h,a,y)
    =
    d_{\mu,t}(h)p_{\mu,t}(a,y\mid h),
\]
where $d_{\mu,t}$ is the distribution of histories visited by the
demonstration process at time $t$.

When the training objective combines decision points from different time
steps, define
\begin{equation}
    p_\mu(h,a,y)
    =
    \sum_{t=0}^{T-1}
    w_t p_{\mu,t}(h,a,y),
    \qquad
    w_t>0,
    \qquad
    \sum_{t=0}^{T-1}w_t=1.
    \label{eq:aggregated-demo-distribution}
\end{equation}
Because time may be included in $H$, this mixture preserves the relevant
time-dependent conditionals.

\begin{definition}[Population risk and population optimum]
\label{def:population-risk}
Let $Z$ be a training example with true distribution $p_\mu$, let
$q\in\cQ$ be a model, and let $\ell(q;Z)$ be its loss. The population risk is
\begin{equation}
    \cL(q)
    :=
    \E_{Z\sim p_\mu}
    [\ell(q;Z)].
    \label{eq:population-risk}
\end{equation}
A population optimum is any
\begin{equation}
    q^*
    \in
    \argminop_{q\in\cQ}
    \cL(q).
    \label{eq:population-optimum}
\end{equation}
\end{definition}

A population objective uses the underlying distribution rather than a finite
sample. For a dataset $\cD_n=\{Z_i\}_{i=1}^n$, the corresponding empirical
risk is
\[
    \widehat{\cL}_n(q)
    =
    \frac{1}{n}
    \sum_{i=1}^n
    \ell(q;Z_i).
\]

A conditional model is exact at the population optimum if it equals the true
conditional distribution for almost every input under the relevant marginal
of $p_\mu$. Equality almost everywhere allows disagreement only on sets of
probability zero.

Define the behavior-action conditional visible to the learner by
\begin{equation}
    \pimu(a\mid h)
    :=
    p_\mu(a\mid h).
    \label{eq:visible-behavior-policy}
\end{equation}
If the demonstrator selects actions using only $H$, then $\pimu$ is the
actual demonstrator policy. If the demonstrator uses private information
omitted from $H$, then $\pimu$ is only the observational action conditional
from the learner's perspective.

\subsection{Control-capability and learner classes}

Let $\EnvClass$ be a family of environments sharing the same history and
action interfaces.

\begin{definition}[Control equivalence]
Two policies $\pi$ and $\pi'$ are control-equivalent over
$(\EnvClass,T)$, written
\[
    \pi\equiv_{\EnvClass,T}\pi',
\]
if
\begin{equation}
    P_{\cM}^{\pi}
    =
    P_{\cM}^{\pi'}
    \qquad
    \text{for every }\cM\in\EnvClass.
    \label{eq:class-control-equivalence}
\end{equation}
\end{definition}

Let $[\pi]_{\EnvClass,T}$ denote the equivalence class containing $\pi$.

\begin{definition}[Control-capability class]
For a policy class $\Pi\subseteq\Kern(\cA\mid\cH)$, define
\begin{equation}
    \CapClass_T(\Pi;\EnvClass)
    :=
    \left\{
        [\pi]_{\EnvClass,T}
        :
        \pi\in\Pi
    \right\}.
    \label{eq:capability-class}
\end{equation}
\end{definition}

The control-capability class records the externally distinguishable
closed-loop behaviors representable by $\Pi$. It does not characterize
sample efficiency, optimization difficulty, parameter count, or utility on a
particular task.

Controller classes must also be distinguished from learner classes. A
controller class describes which policies can be represented; a learner class
describes the information available for selecting one of those policies.

\begin{definition}[Information-restricted population learners]
Let $U:\cT_T\rightarrow[0,1]$ be a known decision objective. An observational
population learner is a mapping
\[
    L_{\mathrm{obs}}:
    \bigl(
        p_\mu(H,A,Y),U
    \bigr)
    \longmapsto
    \pi\in\Pi
\]
that receives the true observational distribution and the utility but no
outcomes generated under additional action interventions. The class of such
learners is denoted by
\[
    \LearnClass_{\mathrm{obs}}(\Pi).
\]

An interventional population learner additionally receives data
$\cD_{\mathrm{int}}$ containing outcomes under specified action
interventions:
\[
    L_{\mathrm{int}}:
    \bigl(
        p_\mu(H,A,Y),
        \cD_{\mathrm{int}},
        U
    \bigr)
    \longmapsto
    \pi\in\Pi.
\]
The corresponding class is denoted by
\[
    \LearnClass_{\mathrm{int}}(\Pi).
\]
\end{definition}

The environment family, utility, and candidate policy class are common to
both learner classes. Their difference is the observational or
interventional information available for selecting a policy.

\subsection{Scope of the comparison}

\begin{assumption}[Population-level comparison]
\label{assumption:population-scope}
Unless otherwise stated, the main equivalence results use the following
conditions:

\begin{enumerate}[label=(\roman*),leftmargin=2em]
    \item \textbf{Common deployment information.}
    The compared policies receive the same history $H$.

    \item \textbf{Common observational data.}
    Direct behavior cloning and world-action training use trajectories from
    the same observational process. The world-action learner may use future
    labels $Y$ contained in those trajectories but receives no additional
    action interventions.

    \item \textbf{Realizability.}
    The relevant policy and conditional-distribution classes contain their
    population targets.

    \item \textbf{Exact population optimization.}
    The population objectives are globally minimized.

    \item \textbf{Distribution-preserving deployment.}
    A world-action controller samples from, or exactly marginalizes over, its
    learned future distribution. Nonlinear point decoding is excluded unless
    explicitly discussed.
\end{enumerate}
\end{assumption}

These conditions remove finite-sample estimation error, approximation error,
and optimization error from the primary comparison. Approximation parameters
introduced later are used only to describe sensitivity of the exact result.

\begin{assumption}[Behavior sufficiency]
\label{assumption:behavior-sufficiency}
Whenever $\pimu$ is identified with the data-generating demonstrator and
equality with the demonstrator's trajectory distribution is claimed, the
demonstration action mechanism is a stochastic kernel of $H$ alone:
\begin{equation}
    A_t
    \sim
    \pimu(\cdot\mid H_t).
    \label{eq:behavior-sufficiency}
\end{equation}
After conditioning on $H_t$, no additional demonstrator-only variable affects
action selection.
\end{assumption}

Behavior sufficiency is not required for the probability identity
\[
    \int
    p_\mu(y\mid h)p_\mu(a\mid h,y)\,dy
    =
    p_\mu(a\mid h).
\]
It is required when $p_\mu(a\mid h)$ is interpreted as the actual policy that
generated the full demonstration trajectory distribution.

\subsection{The three policy-learning paradigms}

\subsubsection{Direct behavior cloning}

Let
\[
    \Pi_{\mathrm{A}}
    \subseteq
    \Kern(\cA\mid\cH)
\]
be a direct policy class. Its population behavior-cloning objective is
\begin{equation}
    \cL_{\mathrm{A}}(\pi_{\mathrm{A}})
    =
    \E_{(H,A)\sim p_\mu}
    \left[
        -\log\pi_{\mathrm{A}}(A\mid H)
    \right].
    \label{eq:bc-loss}
\end{equation}

If $\pimu\in\Pi_{\mathrm{A}}$ and the objective is optimized exactly, then
\begin{equation}
    \pi_{\mathrm{A}}^*(a\mid h)
    =
    p_\mu(a\mid h)
    =
    \pimu(a\mid h)
    \label{eq:bc-target}
\end{equation}
for $p_\mu(H)$-almost every history.

\subsubsection{Imitated world-action policies}

Let
\[
    \cF
    \subseteq
    \Kern(\cY\mid\cH)
\]
be a class of future predictors, and let
\[
    \cI
    \subseteq
    \Kern(\cA\mid\cH\times\cY)
\]
be a class of inverse-action predictors. The future predictor does not
condition on the current candidate action, although $H$ may contain past
actions.

The corresponding world-action joint class is
\begin{equation}
    \cQ_{\mathrm{WA}}
    =
    \left\{
        q_{\mathrm{WA}}(y,a\mid h)
        =
        q_F(y\mid h)q_I(a\mid h,y)
        :
        q_F\in\cF,\ q_I\in\cI
    \right\}.
    \label{eq:wam-joint-class}
\end{equation}

This factorization also covers architectures in which the future and action
are generated by one network, provided that the induced joint distribution
admits the displayed conditionals.

A natural observational objective is
\begin{equation}
\begin{split}
    \cL_{\mathrm{WA}}(q_F,q_I)
    =
    \E_{(H,A,Y)\sim p_\mu}
    \left[
        -\log q_F(Y\mid H)
        -
        \log q_I(A\mid H,Y)
    \right].
\end{split}
\label{eq:wam-loss}
\end{equation}

Under distribution-preserving deployment, the executed policy is
\begin{equation}
    \pi_{\mathrm{WA}}(a\mid h)
    =
    \int_{\cY}
    q_F(y\mid h)q_I(a\mid h,y)\,dy.
    \label{eq:wam-policy}
\end{equation}

The external policy class induced by $(\cF,\cI)$ is
\begin{equation}
\begin{split}
    \Pi_{\mathrm{WA}}(\cF,\cI)
    :=
    \left\{
        \pi:
        \pi(a\mid h)
        =
        \int
        q_F(y\mid h)q_I(a\mid h,y)\,dy,\right.\\
    \left.
        q_F\in\cF,\ q_I\in\cI
    \right\}.
\end{split}
\label{eq:wam-induced-class}
\end{equation}

The joint class $\cQ_{\mathrm{WA}}$ describes internal future-action
representations, whereas $\Pi_{\mathrm{WA}}(\cF,\cI)$ describes externally
observable control behavior.

We use
\begin{equation}
\begin{split}
    \Pi_{\mathrm{WA}}^{\mathrm{all}}
    :=
    \Pi_{\mathrm{WA}}
    \left(
        \Kern(\cY\mid\cH),
        \Kern(\cA\mid\cH\times\cY)
    \right)
\end{split}
\label{eq:unrestricted-wam-class}
\end{equation}
for the induced policy class when both components range over all admissible
stochastic kernels.

\subsubsection{World-model policy learning}

For a specified candidate action $a$, define the interventional outcome
kernel by
\begin{equation}
    \mathsf{T}_a(B\mid h)
    :=
    P_{\cM}
    \bigl(
        Y\in B
        \mid
        H=h,\doop(A=a)
    \bigr),
    \qquad
    B\in\mathcal{B}(\cY),
    \label{eq:interventional-kernel}
\end{equation}
where $\mathcal{B}(\cY)$ denotes the measurable subsets of $\cY$. When the
kernel appears inside an integral, we write
\[
    \mathsf{T}_a(dy\mid h).
\]
Thus, for any bounded measurable function $f$,
\[
    \E\left[
        f(Y)
        \mid
        H=h,\doop(A=a)
    \right]
    =
    \int_{\cY}
    f(y)\mathsf{T}_a(dy\mid h).
\]

Let
\[
    \widehat{\mathsf{T}}
    \in
    \cW_Y
    \subseteq
    \Kern(\cY\mid\cH\times\cA)
\]
be a learned action-conditioned outcome model, where
\[
    \widehat{\mathsf{T}}(B\mid h,a)
\]
approximates $\mathsf{T}_a(B\mid h)$. For a one-step decision problem, $Y$
may be any future variable sufficient to evaluate a bounded utility
$r(h,a,Y)$. The model-based action value is then
\begin{equation}
    \widehat{Q}(h,a)
    :=
    \int_{\cY}
    r(h,a,y)
    \widehat{\mathsf{T}}(dy\mid h,a).
    \label{eq:one-step-model-value}
\end{equation}

For finite-horizon planning, let
\begin{equation}
    \mathsf{K}_{\cM}(C\mid h,a)
    :=
    P_{\cM}
    \bigl(
        H_{t+1}\in C
        \mid
        H_t=h,\doop(A_t=a)
    \bigr),
    \qquad
    C\in\mathcal{B}(\cH),
    \label{eq:true-next-history-kernel}
\end{equation}
denote the true controlled next-history kernel. Under an integral, we write
\[
    \mathsf{K}_{\cM}(dh'\mid h,a).
\]

Let
\[
    \widehat{\mathsf{K}}
    \in
    \cW_H
    \subseteq
    \Kern(\cH\mid\cH\times\cA)
\]
be a learned next-history model. Let $\eta_0$ be the initial-history
distribution. A policy $\pi$ and model $\widehat{\mathsf{K}}$ induce a
rollout according to
\[
    H_0\sim\eta_0,
    \qquad
    A_t\sim\pi(\cdot\mid H_t),
    \qquad
    H_{t+1}
    \sim
    \widehat{\mathsf{K}}(\cdot\mid H_t,A_t).
\]
Denote the resulting trajectory distribution by
$P_{\widehat{\mathsf{K}}}^{\pi}$. For a bounded trajectory utility
$U:\cT_T\rightarrow[0,1]$, define
\begin{equation}
    J_{\widehat{\mathsf{K}},U}^{T}(\pi)
    :=
    \E_{\tau\sim P_{\widehat{\mathsf{K}}}^{\pi}}
    [U(\tau)].
    \label{eq:model-utility}
\end{equation}

Given a candidate policy class $\PiPol$, a world-model policy satisfies
\begin{equation}
    \pi_{\mathrm{WM}}
    \in
    \argmaxop_{\pi\in\PiPol}
    J_{\widehat{\mathsf{K}},U}^{T}(\pi).
    \label{eq:wm-policy}
\end{equation}
Whenever an $\argmaxop$ is displayed, we assume that the maximum is attained.

\begin{table}[t]
\centering
\small
\caption{The three policy-learning paradigms.}
\label{tab:three-policies}
\begin{tabularx}{\textwidth}{
    >{\raggedright\arraybackslash}p{0.12\textwidth}
    >{\raggedright\arraybackslash}p{0.23\textwidth}
    >{\raggedright\arraybackslash}p{0.25\textwidth}
    >{\raggedright\arraybackslash}X
}
\toprule
Policy & Learned object & Population target & Decision principle\\
\midrule
$\pi_{\mathrm{A}}$ &
$\pi_{\mathrm{A}}(a\mid h)$ &
$p_\mu(a\mid h)$ &
Reproduce demonstrated actions\\[0.4em]

$\pi_{\mathrm{WA}}$ &
$q_F(y\mid h)q_I(a\mid h,y)$ &
$p_\mu(y,a\mid h)$ &
Generate a future under the behavior distribution and decode its associated
action\\[0.4em]

$\pi_{\mathrm{WM}}$ &
$\widehat{\mathsf{T}}(dy\mid h,a)$ or
$\widehat{\mathsf{K}}(dh'\mid h,a)$ &
An interventional outcome or transition kernel, when identified &
Compare candidate actions through predicted consequences and a utility\\
\bottomrule
\end{tabularx}
\end{table}

An observational predictor $p(Y\mid H,A)$ is not automatically an
interventional world model. Interpreting it as
$P(Y\mid H,\doop(A))$ requires sufficient action support and a valid causal
identification argument.

% ============================================================
\section{Observational Equivalence of Direct and World-Action Policies}
\label{sec:equivalence}
% ============================================================

We first ask whether the internal future variable enlarges the externally
observable controller class. We then compare the policies selected by the
direct and world-action population objectives.

\subsection{A decision-complete control metric}

Equality under one benchmark reward is too weak to establish general control
equivalence. We instead compare complete trajectory distributions.

\begin{definition}[Control ability]
For a measurable utility $U:\cT_T\rightarrow[0,1]$, define
\begin{equation}
    \Ability_{\cM,U}^{T}(\pi)
    :=
    J_{\cM,U}^{T}(\pi)
    :=
    \E_{\tau\sim P_{\cM}^{\pi}}
    [U(\tau)].
    \label{eq:ability}
\end{equation}
\end{definition}

\begin{definition}[Decision-complete control distance]
For two policies in a fixed environment $\cM$, define
\begin{equation}
\begin{split}
    d_{\mathrm{ctrl}}^{\cM,T}(\pi,\pi')
    :=
    \sup_{U:\cT_T\rightarrow[0,1]}
    \left|
        J_{\cM,U}^{T}(\pi)
        -
        J_{\cM,U}^{T}(\pi')
    \right|.
\end{split}
\label{eq:control-distance}
\end{equation}
\end{definition}

\begin{proposition}[Trajectory characterization]
\label{prop:control-tv}
\begin{equation}
    d_{\mathrm{ctrl}}^{\cM,T}(\pi,\pi')
    =
    \TV\left(
        P_{\cM}^{\pi},
        P_{\cM}^{\pi'}
    \right).
    \label{eq:control-tv}
\end{equation}
\end{proposition}

The proof is given in \cref{app:proof-control-tv}. Zero control distance means
that no bounded return, success indicator, safety statistic, verifier, or
other trajectory-level evaluation can distinguish the policies.

\subsection{Architecture-level flattening}

\begin{theorem}[World-action flattening]
\label{thm:flattening}
For every world-action controller $(q_F,q_I)$, define
\begin{equation}
    \bar{\pi}(a\mid h)
    =
    \int_{\cY}
    q_F(y\mid h)q_I(a\mid h,y)\,dy.
    \label{eq:flattened-policy}
\end{equation}
Then $\bar{\pi}\in\Kern(\cA\mid\cH)$. Let
$P_{\cM}^{(q_F,q_I)}$ denote the trajectory distribution obtained by
deploying the two-stage world-action controller. In every environment with
the same history and action interfaces,
\begin{equation}
    P_{\cM}^{(q_F,q_I)}
    =
    P_{\cM}^{\bar{\pi}}.
    \label{eq:flattened-trajectory-distribution}
\end{equation}

Consequently,
\begin{equation}
    \Pi_{\mathrm{WA}}(\cF,\cI)
    \subseteq
    \Kern(\cA\mid\cH).
    \label{eq:wam-class-inclusion}
\end{equation}

Conversely, if the world-action class permits a point-mass future and
arbitrary inverse kernels at that future, then every direct stochastic policy
has a degenerate world-action representation. Hence
\begin{equation}
\begin{split}
    \CapClass_T
    \bigl(
        \Pi_{\mathrm{WA}}^{\mathrm{all}};
        \EnvClass
    \bigr)
    =
    \CapClass_T
    \bigl(
        \Kern(\cA\mid\cH);
        \EnvClass
    \bigr).
\end{split}
\label{eq:unrestricted-class-equality}
\end{equation}
\end{theorem}

The proof is in \cref{app:proof-flattening}. The theorem shows that the
internal future variable does not enlarge the unrestricted external policy
class. It may still provide a more efficient representation within restricted
parametric classes.

\subsection{Population equivalence under observational training}

Architecture-level flattening does not determine which policy is selected by
training. We next compare the population optima of direct behavior cloning and
world-action maximum likelihood on the same observational distribution.

\begin{assumption}[Ideal observational objectives]
\label{assumption:ideal-learning}
In addition to \cref{assumption:population-scope}, assume that
\[
    \pimu\in\Pi_{\mathrm{A}},
\]
and that the world-action classes contain the true observational
conditionals:
\[
    p_\mu(Y\mid H)\in\cF,
    \qquad
    p_\mu(A\mid H,Y)\in\cI.
\]
\end{assumption}

\begin{theorem}[Population action-marginal equivalence]
\label{thm:population-equivalence}
Under \cref{assumption:ideal-learning}, every exact population optimum
satisfies
\begin{equation}
    \pi_{\mathrm{WA}}^*(a\mid h)
    =
    \pi_{\mathrm{A}}^*(a\mid h)
    =
    \pimu(a\mid h)
    \label{eq:population-equivalence}
\end{equation}
for $p_\mu(H)$-almost every history.

If \cref{assumption:behavior-sufficiency} also holds, then the three action
kernels agree at $d_{\mu,t}$-almost every history for each time $t$, and
\begin{equation}
    P_{\cM}^{\pi_{\mathrm{WA}}^*}
    =
    P_{\cM}^{\pi_{\mathrm{A}}^*}
    =
    P_{\cM}^{\pimu}.
    \label{eq:trajectory-equivalence}
\end{equation}
Consequently,
\begin{equation}
    d_{\mathrm{ctrl}}^{\cM,T}
    \bigl(
        \pi_{\mathrm{WA}}^*,
        \pi_{\mathrm{A}}^*
    \bigr)
    =0.
    \label{eq:zero-distance}
\end{equation}
\end{theorem}

The proof is in \cref{app:proof-population-equivalence}. The central identity
is
\[
\begin{split}
    \pi_{\mathrm{WA}}^*(a\mid h)
    &=
    \int
    p_\mu(y\mid h)
    p_\mu(a\mid h,y)\,dy\\
    &=
    p_\mu(a\mid h).
\end{split}
\]
Thus, sampling a future under the behavior distribution and then sampling its
posterior behavior action recovers the behavior-action conditional.

The theorem concerns distribution-preserving deployment. A system that
selects a single future by MAP decoding, reranks futures with a verifier, or
otherwise modifies the learned joint may induce a different action marginal.

\begin{remark}[Privileged demonstrator information]
If the demonstrator selects actions using a variable omitted from $H$, then
\cref{eq:population-equivalence} still holds as an identity between the
learned action marginals. However, deploying
$\pimu(a\mid h)=p_\mu(a\mid h)$ need not reproduce the original
demonstration trajectory distribution because marginalizing the privileged
variable may destroy action--state dependence. Reproducing the original
demonstrator then requires an augmented history or additional assumptions.
\end{remark}

\subsection{Approximate equivalence}

The exact theorem describes the population limit. The following continuity
bound shows how discrepancies in the learned components translate into
closed-loop disagreement. It does not specify how those discrepancies depend
on sample size.

Suppose that, uniformly over the relevant deployment histories,
\begin{equation}
    \TV\bigl(
        q_F(\cdot\mid h),
        p_\mu(\cdot\mid h)
    \bigr)
    \leq
    \epsilon_F,
    \label{eq:future-tv-error}
\end{equation}
\begin{equation}
    \E_{Y\sim p_\mu(\cdot\mid h)}
    \left[
        \TV\bigl(
            q_I(\cdot\mid h,Y),
            p_\mu(\cdot\mid h,Y)
        \bigr)
    \right]
    \leq
    \epsilon_I,
    \label{eq:inverse-tv-error}
\end{equation}
and
\begin{equation}
    \TV\bigl(
        \pi_{\mathrm{A}}(\cdot\mid h),
        \pimu(\cdot\mid h)
    \bigr)
    \leq
    \epsilon_A.
    \label{eq:direct-tv-error}
\end{equation}

\begin{proposition}[Approximate control equivalence]
\label{prop:approx-equivalence}
Let
\[
    \epsilon
    =
    \min\{1,\epsilon_F+\epsilon_I+\epsilon_A\}.
\]
Then
\begin{equation}
    \TV\bigl(
        \pi_{\mathrm{WA}}(\cdot\mid h),
        \pi_{\mathrm{A}}(\cdot\mid h)
    \bigr)
    \leq
    \epsilon.
    \label{eq:action-approximation}
\end{equation}
If the same bound holds at every history that may be reached by either
policy, then
\begin{equation}
\begin{split}
    d_{\mathrm{ctrl}}^{\cM,T}
    \bigl(
        \pi_{\mathrm{WA}},
        \pi_{\mathrm{A}}
    \bigr)
    &\leq
    1-(1-\epsilon)^T\\
    &\leq
    T\epsilon.
\end{split}
\label{eq:trajectory-approximation}
\end{equation}
\end{proposition}

The proof is in \cref{app:proof-approx-equivalence}.

\subsection{Practical advantages of the world-action factorization}

Equal external capability classes and equal population action targets remain
compatible with substantial practical differences. At training time, the
future $Y$ may contain information about the demonstrated action. For
discrete variables under conditional log loss, the Bayes risks of direct and
future-conditioned action prediction are
\[
    H_\mu(A\mid H)
    \qquad\text{and}\qquad
    H_\mu(A\mid H,Y),
\]
where $H_\mu$ denotes conditional entropy under $p_\mu$. Their difference is
\begin{equation}
    H_\mu(A\mid H)
    -
    H_\mu(A\mid H,Y)
    =
    I_\mu(A;Y\mid H)
    \geq0.
    \label{eq:action-information}
\end{equation}

For Euclidean actions under squared loss,
\begin{equation}
\begin{split}
    &\E\left[
        \|A-\E[A\mid H]\|_2^2
    \right]
    -
    \E\left[
        \|A-\E[A\mid H,Y]\|_2^2
    \right]\\
    &=
    \E\left[
        \|
        \E[A\mid H,Y]-\E[A\mid H]
        \|_2^2
    \right]
    \geq0.
\end{split}
\label{eq:variance-reduction}
\end{equation}

Action prediction may therefore be easier when the realized future is
available as a training-time conditioning variable. At deployment, however,
that future has not yet occurred and must itself be predicted. The
factorization trades lower future-conditioned action uncertainty against
future-prediction error.

Future prediction may also improve video pretraining, temporal
representations, multimodal behavior modeling, parameter sharing, and
cross-task transfer. These are learning and representation advantages; they
do not show that standard future-then-inverse deployment performs
interventional action comparison.

% ============================================================
\section{Interventional Identification and Decision Separation}
\label{sec:separation}
% ============================================================

The previous section rules out an unrestricted external-policy-class
separation. The remaining distinction concerns whether a learner can identify
and use action-specific outcome information.

For the explicit Bayes-ratio arguments in this section, we assume that
$\cA$ and $\cY$ are finite. Accordingly, we write
\[
    \mathsf{T}_a(y\mid h)
    :=
    \mathsf{T}_a(\{y\}\mid h)
\]
for the probability mass assigned to $y$. Integrals with respect to
$\mathsf{T}_a(dy\mid h)$ therefore become sums over $y\in\cY$.

\subsection{Inverse dynamics as a behavior posterior}

Assume temporarily that the observational conditional is causally identified
on the evaluated support:
\begin{equation}
    p_\mu(y\mid h,a)
    =
    \mathsf{T}_a(y\mid h).
    \label{eq:temporary-identification}
\end{equation}

The future predictor then learns the behavior mixture
\begin{equation}
    \mathsf{T}_{\mu}(y\mid h)
    =
    \sum_{a\in\cA}
    \pimu(a\mid h)
    \mathsf{T}_a(y\mid h).
    \label{eq:behavior-mixture}
\end{equation}

For every $y$ with $\mathsf{T}_{\mu}(y\mid h)>0$, the exact inverse
conditional is
\begin{equation}
    p_\mu(a\mid h,y)
    =
    \frac{
        \pimu(a\mid h)\mathsf{T}_a(y\mid h)
    }{
        \sum_{b\in\cA}
        \pimu(b\mid h)\mathsf{T}_b(y\mid h)
    }.
    \label{eq:inverse-posterior}
\end{equation}

The inverse conditional depends on both the environmental effect
$\mathsf{T}_a(y\mid h)$ and the behavior prior $\pimu(a\mid h)$. Summing this
posterior against the behavior-mixture future recovers the behavior policy:
\begin{equation}
    \sum_{y\in\cY}
    \mathsf{T}_\mu(y\mid h)
    p_\mu(a\mid h,y)
    =
    \pimu(a\mid h).
    \label{eq:inverse-recovers-behavior}
\end{equation}

Inverse dynamics asks which behavior action is associated with a realized
transition. An interventional forward model asks what outcome would result
from a specified action. The two conditionals support different deployment
rules.

\subsection{An action-specific prediction metric}

The likelihood of a future under the behavior distribution does not measure
whether a model can answer action-specific ``what if'' questions. We instead
evaluate histories and candidate actions independently.

Let $\nu$ be an evaluation distribution over histories, and let
$\rho(\cdot\mid h)$ be a distribution over candidate actions. Define
\begin{equation}
\begin{split}
    \cE_{\mathrm{int}}(q)
    :=
    \E_{\substack{H\sim\nu\\A\sim\rho(\cdot\mid H)}}
    \left[
        \KL\left(
            \mathsf{T}_A(\cdot\mid H)
            \,\Vert\,
            q(\cdot\mid H,A)
        \right)
    \right].
\end{split}
\label{eq:interventional-risk}
\end{equation}

For a future predictor $g(y\mid h)$ that does not condition on the current
candidate action, define
\[
    q_g(y\mid h,a)
    :=
    g(y\mid h).
\]
Also define
\begin{equation}
    \mathsf{T}_{\rho}(y\mid h)
    =
    \sum_{a\in\cA}
    \rho(a\mid h)
    \mathsf{T}_a(y\mid h).
    \label{eq:rho-mixture}
\end{equation}

For a fixed history $h$, let
\begin{equation}
\begin{split}
    I_{\rho}^{\mathrm{int}}(A;Y\mid H=h)
    :=
    \sum_{a\in\cA}
    \rho(a\mid h)
    \KL\left(
        \mathsf{T}_a(\cdot\mid h)
        \,\Vert\,
        \mathsf{T}_{\rho}(\cdot\mid h)
    \right).
\end{split}
\label{eq:pointwise-interventional-information}
\end{equation}
Averaging over $H\sim\nu$ gives
\begin{equation}
\begin{split}
    I_{\nu,\rho}^{\mathrm{int}}(A;Y\mid H)
    &:=
    \E_{H\sim\nu}
    \left[
        I_{\rho}^{\mathrm{int}}(A;Y\mid H)
    \right].
\end{split}
\label{eq:interventional-mutual-information}
\end{equation}

\begin{theorem}[Irreducible action-unconditioned prediction gap]
\label{thm:prediction-gap}
For every predictor $g(y\mid h)$ that does not condition on the current
candidate action,
\begin{equation}
\begin{split}
    \cE_{\mathrm{int}}(q_g)
    =
    I_{\nu,\rho}^{\mathrm{int}}(A;Y\mid H)
    +
    \E_{H\sim\nu}
    \left[
        \KL\left(
            \mathsf{T}_{\rho}(\cdot\mid H)
            \,\Vert\,
            g(\cdot\mid H)
        \right)
    \right].
\end{split}
\label{eq:prediction-gap}
\end{equation}
Consequently,
\begin{equation}
    \inf_g
    \cE_{\mathrm{int}}(q_g)
    =
    I_{\nu,\rho}^{\mathrm{int}}(A;Y\mid H),
    \label{eq:minimum-gap}
\end{equation}
whereas an exact action-conditioned model
$q(y\mid h,a)=\mathsf{T}_a(y\mid h)$ has zero risk.
\end{theorem}

The proof is in \cref{app:proof-prediction-gap}. The lower bound is positive
whenever candidate actions produce different outcome distributions on a set
of histories with positive $\nu$-probability. A future under the behavior
distribution therefore cannot substitute for an action-specific prediction.

\subsection{Converting a world-action joint into a forward model}
\label{sec:wam-forward-conversion}

The preceding theorem concerns the future predictor in isolation. The complete
world-action joint may contain more information. An exact joint determines
the observational conditional $p_\mu(Y\mid H,A)$ wherever the behavior
distribution assigns positive probability to the action. Whether that
conditional has a causal interpretation, and whether the controller uses it
for action comparison, are separate questions.

\begin{proposition}[Extracting an observational forward conditional]
\label{prop:wam-forward-conversion}
Suppose that the learned world-action joint is exact at the population
optimum:
\begin{equation}
    q_F^*(y\mid h)q_I^*(a\mid h,y)
    =
    p_\mu(y,a\mid h)
    \label{eq:exact-wam-joint}
\end{equation}
for $p_\mu(H)$-almost every $h$. Define
\begin{equation}
    \bar{\pi}_q(a\mid h)
    :=
    \sum_{y\in\cY}
    q_F^*(y\mid h)q_I^*(a\mid h,y).
    \label{eq:wam-induced-action-marginal}
\end{equation}
Then
\begin{equation}
    \bar{\pi}_q(a\mid h)
    =
    p_\mu(a\mid h).
    \label{eq:wam-induced-action-equality}
\end{equation}

For every pair $(h,a)$ with $\bar{\pi}_q(a\mid h)>0$, define
\begin{equation}
    \mathsf{T}^{\mathrm{WA}}_q(y\mid h,a)
    :=
    \frac{
        q_F^*(y\mid h)q_I^*(a\mid h,y)
    }{
        \bar{\pi}_q(a\mid h)
    }.
    \label{eq:wam-extracted-forward-model}
\end{equation}
Then
\begin{equation}
    \mathsf{T}^{\mathrm{WA}}_q(y\mid h,a)
    =
    p_\mu(y\mid h,a).
    \label{eq:wam-extracted-observational-conditional}
\end{equation}
\end{proposition}

\begin{proof}
Summing \cref{eq:exact-wam-joint} over $y$ gives
\[
\begin{split}
    \bar{\pi}_q(a\mid h)
    &=
    \sum_{y\in\cY}
    q_F^*(y\mid h)q_I^*(a\mid h,y)\\
    &=
    \sum_{y\in\cY}
    p_\mu(y,a\mid h)\\
    &=
    p_\mu(a\mid h).
\end{split}
\]
For $p_\mu(a\mid h)>0$,
\[
\begin{split}
    \mathsf{T}^{\mathrm{WA}}_q(y\mid h,a)
    &=
    \frac{
        q_F^*(y\mid h)q_I^*(a\mid h,y)
    }{
        \bar{\pi}_q(a\mid h)
    }\\
    &=
    \frac{
        p_\mu(y,a\mid h)
    }{
        p_\mu(a\mid h)
    }\\
    &=
    p_\mu(y\mid h,a).
\end{split}
\]
\end{proof}

The proposition is probabilistic rather than causal. To identify
$p_\mu(Y\mid H,A)$ with an interventional distribution, let $Y(a)$ denote the
potential outcome under action $a$ and suppose that the following conditions
hold:

\begin{enumerate}[leftmargin=2em]
    \item \textbf{Consistency:}
    if $A=a$, then $Y=Y(a)$;

    \item \textbf{Conditional exchangeability:}
    \begin{equation}
        Y(a)\indep A\mid H;
        \label{eq:conditional-exchangeability}
    \end{equation}

    \item \textbf{Positivity:}
    \begin{equation}
        p_\mu(a\mid h)>0
        \label{eq:causal-positivity}
    \end{equation}
    for every evaluated history-action pair.
\end{enumerate}

Under these conditions,
\begin{equation}
\begin{split}
    \mathsf{T}^{\mathrm{WA}}_q(y\mid h,a)
    &=
    p_\mu(y\mid h,a)\\
    &=
    P_{\cM}
    \bigl(
        Y=y
        \mid
        H=h,\doop(A=a)
    \bigr)\\
    &=
    \mathsf{T}_a(y\mid h).
\end{split}
\label{eq:wam-extracted-interventional-model}
\end{equation}

Equivalently,
\begin{equation}
    \mathsf{T}_a(y\mid h)
    =
    \frac{
        q_F^*(y\mid h)q_I^*(a\mid h,y)
    }{
        \sum_{y'\in\cY}
        q_F^*(y'\mid h)q_I^*(a\mid h,y')
    }.
    \label{eq:wam-interventional-recovery}
\end{equation}
At the population optimum,
\begin{equation}
    \mathsf{T}_a(y\mid h)
    =
    \frac{
        p_\mu(y\mid h)p_\mu(a\mid h,y)
    }{
        p_\mu(a\mid h)
    }.
    \label{eq:bayes-recovery}
\end{equation}
The causal derivation is given in \cref{app:proof-bayes-recovery}.

Three requirements are involved: the world-action factors must recover the
observational joint, the action must have positive support so that the
observational conditional is defined, and that conditional must be causally
identified. None of these requirements determines how the model is deployed.

Indeed, standard world-action deployment still gives
\begin{equation}
\begin{split}
    \pi_{\mathrm{WA}}^*(a\mid h)
    &=
    \sum_{y\in\cY}
    q_F^*(y\mid h)q_I^*(a\mid h,y)\\
    &=
    p_\mu(a\mid h).
\end{split}
\label{eq:standard-wam-still-behavior}
\end{equation}

A planning rule uses the same joint differently. For a one-step utility
$r(h,a,y)\in[0,1]$, define
\begin{equation}
    Q_{\mathrm{WA}}(h,a)
    :=
    \E_{
        Y\sim
        \mathsf{T}^{\mathrm{WA}}_q(\cdot\mid h,a)
    }
    \left[
        r(h,a,Y)
    \right]
    \label{eq:wam-converted-q}
\end{equation}
and select
\begin{equation}
    \pi_{\mathrm{plan}}(h)
    \in
    \argmaxop_a
    Q_{\mathrm{WA}}(h,a).
    \label{eq:wam-converted-planner}
\end{equation}

\begin{corollary}[One-step planning with an identified world-action joint]
\label{cor:planning-with-wam-joint}
Suppose that \cref{prop:wam-forward-conversion} holds and that consistency,
conditional exchangeability, and positivity identify
\[
    \mathsf{T}^{\mathrm{WA}}_q(\cdot\mid h,a)
    =
    \mathsf{T}_a(\cdot\mid h)
\]
for every action considered by the planner. Then
\begin{equation}
    \pi_{\mathrm{plan}}(h)
    \in
    \argmaxop_a
    \E_{Y\sim\mathsf{T}_a(\cdot\mid h)}
    [r(h,a,Y)].
    \label{eq:identified-wam-optimal-planning}
\end{equation}
\end{corollary}

\begin{proof}
Under the assumptions,
\[
    Q_{\mathrm{WA}}(h,a)
    =
    \E_{Y\sim\mathsf{T}_a(\cdot\mid h)}
    [r(h,a,Y)]
\]
for every candidate action. Maximizing $Q_{\mathrm{WA}}(h,a)$ therefore
maximizes the true interventional expected utility.
\end{proof}

The same exact joint can thus support either behavior reproduction or
action comparison. The distinction lies in causal identification and
deployment, not merely in the represented joint distribution.

\subsection{Observational non-identification}

The recovery result depends on support and causal assumptions. Observational
fitting alone does not guarantee either condition.

\begin{theorem}[Two sources of observational non-identification]
\label{thm:nonidentification}
Interventional action effects are not identified by
$p_\mu(H,A,Y)$ in general.

\begin{enumerate}[label=(\roman*),leftmargin=2em]
    \item Without positivity, two causal environments can have the same
    observational distribution while differing under an unsupported action
    intervention.

    \item Even when every observed action has positive probability, two
    causal environments can have the same observational distribution while
    differing interventionally if $H$ omits a variable that affects both
    action selection and outcomes.
\end{enumerate}
\end{theorem}

The proof is in \cref{app:proof-nonidentification}. The first construction
uses an action that is never selected by the behavior policy. The second uses
a demonstrator-only variable omitted from $H$. Thus, action coverage and
causal sufficiency address different identification failures.

\subsection{Strict separation of observational and interventional learners}

The strongest separation is not between the external policy classes. It is
between learners that receive different information about action effects.

For an environment $\cM$ and bounded utility $U$, define regret relative to a
candidate class $\PiPol$ by
\begin{equation}
    \Reg_{\cM,U}(\pi)
    :=
    \sup_{\pi'\in\PiPol}
    J_{\cM,U}^{T}(\pi')
    -
    J_{\cM,U}^{T}(\pi).
    \label{eq:policy-regret}
\end{equation}

Under the assumptions of \cref{thm:population-equivalence},
\[
    \pi_{\mathrm{WA}}^*
    =
    \pi_{\mathrm{A}}^*
    =
    \pimu.
\]

\begin{corollary}[Suboptimal-demonstrator separation]
\label{cor:suboptimal-separation}
Suppose that behavior sufficiency holds, that $\pimu\in\PiPol$, and that
\begin{equation}
    \Delta
    :=
    \sup_{\pi\in\PiPol}
    J_{\cM,U}^{T}(\pi)
    -
    J_{\cM,U}^{T}(\pimu)
    >0.
    \label{eq:demonstrator-gap}
\end{equation}
Assume that an optimal policy in $\PiPol$ exists, that the learned
next-history model equals the true interventional kernel on every
history-action pair reachable by policies in $\PiPol$, and that model-based
optimization is exact. Then
\begin{equation}
\begin{split}
    J_{\cM,U}^{T}(\pi_{\mathrm{WM}}^*)
    &=
    J_{\cM,U}^{T}(\pi_{\mathrm{WA}}^*)+\Delta\\
    &=
    J_{\cM,U}^{T}(\pi_{\mathrm{A}}^*)+\Delta.
\end{split}
\label{eq:suboptimal-separation}
\end{equation}
\end{corollary}

\begin{proof}
Exactness of the interventional model on the relevant history-action pairs
implies
\[
    J_{\widehat{\mathsf{K}},U}^{T}(\pi)
    =
    J_{\cM,U}^{T}(\pi)
    \qquad
    \text{for every }\pi\in\PiPol.
\]
Exact optimization therefore gives the best true value in $\PiPol$. By
\cref{thm:population-equivalence} and behavior sufficiency,
\[
    J_{\cM,U}^{T}(\pi_{\mathrm{WA}}^*)
    =
    J_{\cM,U}^{T}(\pi_{\mathrm{A}}^*)
    =
    J_{\cM,U}^{T}(\pimu).
\]
Substituting the definition of $\Delta$ proves the result.
\end{proof}

The corollary assumes that the relevant action effects are already known or
identified. Interventions can also create a strict information advantage.

\begin{theorem}[Strict value of interventional information]
\label{thm:strict-intervention}
There exists a one-step, two-action environment family with a common known
utility and candidate class
\[
    \PiPol
    =
    \Kern(\{0,1\}\mid\{h\})
\]
such that:

\begin{enumerate}[label=(\roman*),leftmargin=2em]
    \item the observational distribution $p_\mu(H,A,Y)$ is identical across
    the environments;

    \item every observational population learner
    $L_{\mathrm{obs}}\in\LearnClass_{\mathrm{obs}}(\PiPol)$ has worst-case
    regret at least $1/4$;

    \item the exact population optima $\pi_{\mathrm{A}}^*$ and
    $\pi_{\mathrm{WA}}^*$ have worst-case regret $1/2$;

    \item one informative action intervention identifies the environment and
    permits an interventional learner, representable as an exact
    world-model learner, to achieve zero regret.
\end{enumerate}
\end{theorem}

The proof is in \cref{app:proof-strict-intervention}. The theorem isolates the
source of the separation: the observational learners output ordinary
stochastic policies, but the observational distribution does not determine
which policy is optimal across the environment family.

% ============================================================
\section{Conclusion}
\label{sec:limitations-conclusion}
% ============================================================
A direct behavior-cloning policy and an imitation-trained world-action policy
use different internal factorizations, but every distributional world-action
controller induces an ordinary stochastic action kernel. Under unrestricted
kernel classes, the two therefore have the same external
control-capability class. Under realizability, exact population optimization,
matched deployment information, and distribution-preserving deployment, they
also have the same population action target:
\[
    \pi_{\mathrm{WA}}^*
    =
    \pi_{\mathrm{A}}^*
    =
    \pimu.
\]
With behavior sufficiency, this equality extends to their complete
closed-loop trajectory distributions.

These results do not deny the practical value of world-action modeling.
Future prediction may improve temporal representations, video pretraining,
multimodal behavior modeling, parameter sharing, and finite-sample
performance. Gains of this kind concern how the observational target is
represented and learned; they do not by themselves establish interventional
reasoning.

The complete world-action joint can contain more information than the future
predictor alone. With positive action support, an exact joint determines
$p_\mu(Y\mid H,A)$. Under additional causal assumptions, that conditional may
equal
\[
    P(Y\mid H,\doop(A=a)).
\]
Standard future-then-inverse deployment nevertheless marginalizes the joint
and reproduces the behavior policy. Using the same representation for
world-model control requires a different operation: candidate actions must be
specified, their consequences evaluated, and their utilities compared.

Observational demonstrations do not identify these action effects in general.
Unsupported actions and hidden confounding can make different causal
environments observationally indistinguishable. Interventions, exploratory
coverage, or valid causal assumptions are required whenever the desired
decision depends on effects not identified by the demonstration
distribution.

This distinction also suggests how empirical claims should be evaluated.
Imitation quality should be separated from action-effect prediction. Evidence
for interventional control should test predictions from the same state under
multiple specified actions and determine whether explicit action comparison
improves value beyond future-then-inverse deployment. A practical system may
share a common representation while using separate components for
\[
    q_F(Y\mid H),
    \qquad
    q_I(A\mid H,Y),
    \qquad
    q_{\mathrm{do}}(Y\mid H,A).
\]
The first two support structured imitation; the third supports
action-specific evaluation.

The essential distinction is therefore between predicting a future associated
with observed behavior and predicting the consequences of specified actions
for policy optimization:
\[
    \text{future prediction for behavior decoding}
    \;\neq\;
    \text{interventional action-conditioned prediction for control}.
\]

While we expect this paper to shed some light on the understanding of world-action models and world models, it comes with several obvious limitations:

\textbf{Class-level rather than finite-network equivalence.}
The flattening theorem concerns unrestricted stochastic kernels, or
restricted classes closed under the required marginalization. Under
parameter, memory, latency, or compute constraints, a world-action
factorization may represent a useful policy more efficiently than a direct
architecture.

\textbf{Population rather than statistical analysis.}
The main comparison assumes realizability and exact population optimization.
It does not provide sample-complexity bounds or optimization guarantees for
large neural models. The approximate result only describes sensitivity to
specified distributional errors.

\textbf{Matched deployment information.}
The equivalence results require the compared policies to receive the same
history. If one architecture receives a longer context, privileged state,
future frame, external memory, or additional instruction at deployment, the
policy classes are not directly comparable.

\textbf{Behavior sufficiency.}
Equality with the demonstrator's trajectory distribution requires the
recorded history to contain the information used for action selection. When a
demonstrator uses omitted private information, direct and world-action
training still recover the same observational action marginal, but deploying
that marginal need not reproduce the original demonstrator--environment
coupling.

\textbf{Distribution-preserving deployment.}
MAP future selection, deterministic decoding, temperature changes,
best-of-$N$ sampling, and verifier-based future selection may alter the action
marginal. Such controllers remain flattenable into direct policies, but they
need not equal the behavior-cloning population optimum.

\textbf{Conditional identification.}
A perfect observational joint identifies interventional dynamics only under
valid support and causal assumptions. Greater model capacity cannot
substitute for missing interventions or unrecorded confounders.

\textbf{Decision relevance of the future variable.}
The predicted variable $Y$ must preserve information relevant to the utility
and subsequent decisions. Accurate prediction of visually salient but
decision-irrelevant features does not guarantee useful control.

\section*{Acknowledgments}

The author used a large language model as a writing assistant during the
preparation of this manuscript. The LLM assisted with language editing, symbol consistency checking, and formatting.

% ============================================================
\appendix
% ============================================================

\section{Proofs}
\label{app:proofs}

\subsection{Proof of \cref{prop:control-tv}}
\label{app:proof-control-tv}

For probability measures $P$ and $Q$,
\[
    \TV(P,Q)
    =
    \sup_{0\leq f\leq1}
    \left|
        \E_P[f]-\E_Q[f]
    \right|.
\]
Applying this variational characterization to
$P_{\cM}^{\pi}$ and $P_{\cM}^{\pi'}$, with $f=U$, gives
\[
\begin{split}
    \TV\left(
        P_{\cM}^{\pi},
        P_{\cM}^{\pi'}
    \right)
    &=
    \sup_{U:\cT_T\rightarrow[0,1]}
    \left|
        \E_{\tau\sim P_{\cM}^{\pi}}[U(\tau)]
        -
        \E_{\tau\sim P_{\cM}^{\pi'}}[U(\tau)]
    \right|\\
    &=
    d_{\mathrm{ctrl}}^{\cM,T}(\pi,\pi').
\end{split}
\]

\subsection{Proof of \cref{thm:flattening}}
\label{app:proof-flattening}

Define
\[
    \bar{\pi}(a\mid h)
    =
    \int
    q_F(y\mid h)q_I(a\mid h,y)\,dy.
\]
By construction,
\[
    \bar{\pi}(\cdot\mid h)
    =
    \pi_{\mathrm{WA}}(\cdot\mid h)
\]
at every history.

Starting from the same initial-history distribution, the two controllers
induce the same action distribution conditional on every common history.
Because they also share the same environment transition and observation
kernels, they induce the same next-history distribution. Induction over time
gives
\[
    P_{\cM}^{(q_F,q_I)}
    =
    P_{\cM}^{\bar{\pi}}.
\]

Conversely, let $\pi\in\Kern(\cA\mid\cH)$ and choose $y_0\in\cY$. Let
$\delta_{y_0}$ be the Dirac probability measure concentrated at $y_0$, and
define
\[
    q_F(\cdot\mid h)
    =
    \delta_{y_0},
    \qquad
    q_I(a\mid h,y_0)
    =
    \pi(a\mid h).
\]
Then
\[
\begin{split}
    \pi_{\mathrm{WA}}(a\mid h)
    &=
    \int
    \delta_{y_0}(dy)
    q_I(a\mid h,y)\\
    &=
    \pi(a\mid h).
\end{split}
\]
Thus, every direct stochastic policy has a degenerate world-action
representation.

\subsection{Proof of \cref{thm:population-equivalence}}
\label{app:proof-population-equivalence}

Let $H_\mu$ denote conditional entropy under the demonstration distribution.
The direct behavior-cloning objective decomposes as
\[
\begin{split}
    \cL_{\mathrm{A}}(\pi_{\mathrm{A}})
    &=
    H_\mu(A\mid H)\\
    &\quad+
    \E_H
    \left[
        \KL\left(
            p_\mu(A\mid H)
            \,\Vert\,
            \pi_{\mathrm{A}}(A\mid H)
        \right)
    \right].
\end{split}
\]
Under realizability, the KL term can be minimized to zero, so
\[
    \pi_{\mathrm{A}}^*(a\mid h)
    =
    p_\mu(a\mid h)
    =
    \pimu(a\mid h)
\]
for $p_\mu(H)$-almost every $h$.

Similarly,
\[
\begin{split}
    \cL_{\mathrm{WA}}(q_F,q_I)
    &=
    H_\mu(Y\mid H)
    +
    H_\mu(A\mid H,Y)\\
    &\quad+
    \E_H
    \left[
        \KL\left(
            p_\mu(Y\mid H)
            \,\Vert\,
            q_F(Y\mid H)
        \right)
    \right]\\
    &\quad+
    \E_{H,Y}
    \left[
        \KL\left(
            p_\mu(A\mid H,Y)
            \,\Vert\,
            q_I(A\mid H,Y)
        \right)
    \right].
\end{split}
\]
At an exact population optimum,
\[
    q_F^*(y\mid h)
    =
    p_\mu(y\mid h),
    \qquad
    q_I^*(a\mid h,y)
    =
    p_\mu(a\mid h,y)
\]
almost everywhere. Therefore,
\[
\begin{split}
    \pi_{\mathrm{WA}}^*(a\mid h)
    &=
    \int
    p_\mu(y\mid h)
    p_\mu(a\mid h,y)\,dy\\
    &=
    \int
    p_\mu(a,y\mid h)\,dy\\
    &=
    p_\mu(a\mid h)\\
    &=
    \pimu(a\mid h).
\end{split}
\]

Under behavior sufficiency, $d_{\mu,t}$ is generated by deploying $\pimu$ in
$\cM$. At $d_{\mu,t}$-almost every history, the three policies have the same
action kernel. Induction over time yields
\[
    P_{\cM}^{\pi_{\mathrm{WA}}^*}
    =
    P_{\cM}^{\pi_{\mathrm{A}}^*}
    =
    P_{\cM}^{\pimu}.
\]

\subsection{Proof of \cref{prop:approx-equivalence}}
\label{app:proof-approx-equivalence}

Fix a history $h$ and define
\[
    Q_h(dy,da)
    =
    q_F(dy\mid h)q_I(da\mid h,y),
\]
\[
    P_h(dy,da)
    =
    p_\mu(dy\mid h)p_\mu(da\mid h,y),
\]
and
\[
    R_h(dy,da)
    =
    p_\mu(dy\mid h)q_I(da\mid h,y).
\]

Marginalization contracts total variation, so
\[
    \TV\left(
        \pi_{\mathrm{WA}}(\cdot\mid h),
        \pimu(\cdot\mid h)
    \right)
    \leq
    \TV(Q_h,P_h).
\]
By the triangle inequality,
\[
    \TV(Q_h,P_h)
    \leq
    \TV(Q_h,R_h)
    +
    \TV(R_h,P_h).
\]
The two terms are bounded by $\epsilon_F$ and $\epsilon_I$, respectively.
Hence
\[
    \TV\left(
        \pi_{\mathrm{WA}}(\cdot\mid h),
        \pimu(\cdot\mid h)
    \right)
    \leq
    \epsilon_F+\epsilon_I.
\]
Applying the triangle inequality through $\pimu$ gives
\[
    \TV\left(
        \pi_{\mathrm{WA}}(\cdot\mid h),
        \pi_{\mathrm{A}}(\cdot\mid h)
    \right)
    \leq
    \epsilon.
\]

For the sequential bound, use maximal coupling at each common history. As
long as the trajectory prefixes agree, the actions can be coupled to agree
with probability at least $1-\epsilon$. Conditional on equal actions, the
same environment kernel can be used to couple the next observations.
Therefore,
\[
    \TV\left(
        P_{\cM}^{\pi_{\mathrm{WA}}},
        P_{\cM}^{\pi_{\mathrm{A}}}
    \right)
    \leq
    1-(1-\epsilon)^T
    \leq
    T\epsilon.
\]
The conclusion follows from \cref{prop:control-tv}.

\subsection{Proof of \cref{thm:prediction-gap}}
\label{app:proof-prediction-gap}

For a fixed history $h$,
\[
\begin{split}
    &\sum_{a\in\cA}
    \rho(a\mid h)
    \KL\left(
        \mathsf{T}_a(\cdot\mid h)
        \,\Vert\,
        g(\cdot\mid h)
    \right)\\
    &=
    \sum_{a\in\cA}
    \rho(a\mid h)
    \sum_{y\in\cY}
    \mathsf{T}_a(y\mid h)
    \log
    \frac{
        \mathsf{T}_a(y\mid h)
    }{
        g(y\mid h)
    }.
\end{split}
\]
Insert $\mathsf{T}_\rho$:
\[
    \log
    \frac{\mathsf{T}_a}{g}
    =
    \log
    \frac{\mathsf{T}_a}{\mathsf{T}_\rho}
    +
    \log
    \frac{\mathsf{T}_\rho}{g}.
\]
The first term gives
\[
    I_{\rho}^{\mathrm{int}}
    (A;Y\mid H=h).
\]
Using
\[
    \sum_{a\in\cA}
    \rho(a\mid h)
    \mathsf{T}_a(y\mid h)
    =
    \mathsf{T}_\rho(y\mid h),
\]
the second term gives
\[
    \KL\left(
        \mathsf{T}_\rho(\cdot\mid h)
        \,\Vert\,
        g(\cdot\mid h)
    \right).
\]
Averaging over $H\sim\nu$ proves \cref{eq:prediction-gap}. The second KL term
is minimized to zero by
\[
    g(\cdot\mid h)
    =
    \mathsf{T}_\rho(\cdot\mid h),
\]
which proves \cref{eq:minimum-gap}.

\subsection{Derivation of \cref{eq:bayes-recovery}}
\label{app:proof-bayes-recovery}

Conditional exchangeability gives
\[
    Y(a)\indep A\mid H.
\]
Therefore,
\[
    P(Y(a)=y\mid H=h)
    =
    P(Y(a)=y\mid H=h,A=a).
\]
By consistency,
\[
    P(Y(a)=y\mid H=h,A=a)
    =
    P(Y=y\mid H=h,A=a).
\]
Thus,
\[
    \mathsf{T}_a(y\mid h)
    =
    p_\mu(y\mid h,a).
\]
Under positivity,
\[
\begin{split}
    p_\mu(y\mid h,a)
    &=
    \frac{
        p_\mu(y,a\mid h)
    }{
        p_\mu(a\mid h)
    }\\
    &=
    \frac{
        p_\mu(y\mid h)p_\mu(a\mid h,y)
    }{
        \pimu(a\mid h)
    }.
\end{split}
\]

\subsection{Proof of \cref{thm:nonidentification}}
\label{app:proof-nonidentification}

We give separate constructions for support failure and hidden confounding.

\paragraph{Support failure.}
Let $H$ be constant and let $\cA=\{0,1\}$. The behavior policy always selects
$A=0$. Define two environments by
\[
\begin{array}{c|cc}
 & Y(0) & Y(1)\\
\hline
\cM_1 & 0 & 0\\
\cM_2 & 0 & 1.
\end{array}
\]
Both environments produce
\[
    P(A=0,Y=0)=1
\]
under the behavior policy, so their observational distributions are
identical. Under $\doop(A=1)$,
\[
    P_{\cM_1}(Y=1\mid\doop(A=1))=0,
\]
whereas
\[
    P_{\cM_2}(Y=1\mid\doop(A=1))=1.
\]

\paragraph{Hidden confounding under positive action support.}
Let $H$ be constant and let
\[
    Z\sim\operatorname{Bernoulli}(1/2)
\]
be observed by the demonstrator but omitted from $H$. Let the demonstrator
choose
\[
    A=Z.
\]
Both actions have positive observational probability.

In environment $\cM_1$, define
\[
    Y=A.
\]
In environment $\cM_2$, define
\[
    Y=Z.
\]
Under the demonstration mechanism, both environments produce
\[
    P(A=0,Y=0)
    =
    \frac12,
    \qquad
    P(A=1,Y=1)
    =
    \frac12.
\]
Their observational distributions are therefore identical. Under
$\doop(A=1)$,
\[
    P_{\cM_1}(Y=1\mid\doop(A=1))
    =
    1,
\]
whereas
\[
    P_{\cM_2}(Y=1\mid\doop(A=1))
    =
    P(Z=1)
    =
    \frac12.
\]
Positive observational action support is therefore insufficient when $H$
omits an action--outcome confounder.

\subsection{Proof of \cref{thm:strict-intervention}}
\label{app:proof-strict-intervention}

Consider a one-step problem with a single history and action space
$\{0,1\}$. The behavior policy always selects
\[
    A=0.
\]
Define two environments:
\[
\begin{array}{c|cc}
 & Y(0) & Y(1)\\
\hline
\cM_{+} & 1/2 & 1\\
\cM_{-} & 1/2 & 0.
\end{array}
\]
Let the common known utility be
\[
    U=Y.
\]

The observational demonstration is identical in the two environments:
\[
    A=0,
    \qquad
    Y=\frac12.
\]
Every observational learner therefore receives the same observational
distribution and utility in both environments and must produce the same
stochastic action distribution.

Let $p$ be the probability of selecting $A=1$. In $\cM_+$, the regret is
\[
    \Reg_+(p)
    =
    \frac12(1-p).
\]
In $\cM_-$, the regret is
\[
    \Reg_-(p)
    =
    \frac{p}{2}.
\]
Therefore,
\[
\begin{split}
    \max\{
        \Reg_+(p),
        \Reg_-(p)
    \}
    &=
    \frac12
    \max\{1-p,p\}\\
    &\geq
    \frac14.
\end{split}
\]

The exact behavior-cloning and world-action population optima select $A=0$
with probability one, so $p=0$. Their worst-case regret is $1/2$.

Finally, one intervention selecting $A=1$ yields
\[
    Y=1
    \quad\text{in }\cM_+,
\]
and
\[
    Y=0
    \quad\text{in }\cM_-.
\]
The intervention identifies the environment. Selecting $A=1$ in $\cM_+$ and
$A=0$ in $\cM_-$ then achieves zero regret.

% ============================================================
\section{Additional Technical Remarks}
\label{app:technical-remarks}
% ============================================================

\subsection{Point decoding}
\label{app:decoding}

Suppose that, for a fixed history,
\[
    P(Y=0,A=0)=0.6,
    \qquad
    P(Y=1,A=1)=0.4.
\]
A distributional world-action model reproduces
\[
    P(A=0)=0.6,
    \qquad
    P(A=1)=0.4.
\]
A MAP-future decoder instead selects $Y=0$ and always executes $A=0$. The
resulting controller can still be flattened into a direct deterministic
policy, but it is not equal to the distributional behavior-cloning optimum.

For continuous actions under squared loss,
\[
    \E_{Y\mid H=h}
    \left[
        \E[A\mid H=h,Y]
    \right]
    =
    \E[A\mid H=h]
\]
by the tower property. Sampling one future and applying its conditional mean
preserves equality in expectation but not necessarily equality of the full
action distribution.

\subsection{Privileged demonstrator information}

The hidden-confounding construction in
\cref{app:proof-nonidentification} also shows why matching
$p_\mu(A\mid H)$ need not reproduce the original demonstration trajectory
distribution when behavior sufficiency fails.

In that construction,
\[
    \pimu(A=1\mid H)
    =
    \frac12,
\]
but the demonstrated action satisfies $A=Z$. Deploying an independent
Bernoulli action with the same marginal does not reproduce the dependence
between $A$ and the omitted variable $Z$. Including $Z$ in the history,
\[
    H'=(H,Z),
\]
restores the information used by the demonstrator.

\subsection{Action chunks and language conditioning}

If
\[
    A=A_{t:t+K-1}
\]
is an action chunk and
\[
    Y=Y_{t+1:t+K}
\]
is the corresponding future chunk, the same marginalization identity holds:
\[
\begin{split}
    &\int
    p_\mu(
        y_{t+1:t+K}
        \mid
        h_t
    )
    p_\mu(
        a_{t:t+K-1}
        \mid
        h_t,y_{t+1:t+K}
    )
    \,dy\\
    &=
    p_\mu(
        a_{t:t+K-1}
        \mid
        h_t
    ).
\end{split}
\]

Goals and language instructions are handled by including $G$ in the history.
The results then apply conditionally on $G$, provided that the compared
policies receive the same instruction and context at deployment.

\subsection{Restricted parametric classes}

Let $\Pi_{\mathrm{A}}^{\mathrm{param}}$ be a restricted direct policy class and
let
\[
    \Pi_{\mathrm{WA}}^{\mathrm{param}}
    =
    \Pi_{\mathrm{WA}}(\cF,\cI).
\]
The flattening theorem guarantees
\[
    \Pi_{\mathrm{WA}}^{\mathrm{param}}
    \subseteq
    \Kern(\cA\mid\cH),
\]
but not necessarily
\[
    \Pi_{\mathrm{WA}}^{\mathrm{param}}
    \subseteq
    \Pi_{\mathrm{A}}^{\mathrm{param}}.
\]

A restricted direct class reproduces the world-action class only if it is
closed under the required marginalization:
\[
    q_F\in\cF,\ q_I\in\cI
    \quad\Longrightarrow\quad
    \left[
        h\mapsto
        \int
        q_F(y\mid h)q_I(\cdot\mid h,y)\,dy
    \right]
    \in
    \Pi_{\mathrm{A}}^{\mathrm{param}}.
\]
Finite architectures may therefore exhibit a genuine representational or
computational advantage from the world-action factorization even though no
such advantage exists relative to unrestricted stochastic policies.

\subsection{Observational recovery and model exploitation}

Even when
\[
    p_\mu(Y\mid H,A)
    =
    P(Y\mid H,\doop(A))
\]
on the behavior support, an optimized policy may select actions or reach
histories outside that support. Its performance then depends on extrapolation
of the learned model. Causal identification on the observed support and
accuracy under the optimized policy are different requirements.

This distinction motivates support constraints, pessimistic or robust
planning, uncertainty-aware action selection, and online model correction.

% ============================================================
% Bibliography
% ============================================================

\bibliographystyle{plainnat}
\bibliography{ref}

\end{document}